\documentclass{article}
\usepackage{arxiv}

\usepackage[utf8]{inputenc} 
\usepackage[T1]{fontenc}    
\usepackage{hyperref}       
\usepackage{url}            
\usepackage{booktabs}       
\usepackage{subcaption}
\usepackage{amsmath}
\usepackage{amsthm}
\usepackage{amsfonts}       
\usepackage{nicefrac}       
\usepackage{microtype}      
\usepackage{lipsum}
\usepackage{fancyhdr}       
\usepackage{graphicx}       
\graphicspath{{media/}}     

\usepackage{hyperref}       
\usepackage{url}            
\usepackage{booktabs}       
\usepackage{amsfonts}       
\usepackage{nicefrac}       
\usepackage{microtype}      
\usepackage{bm}
\usepackage{graphicx}
\usepackage{subcaption}
\usepackage{amsmath}
\usepackage{amsthm}
\usepackage{enumitem}
\usepackage[table]{xcolor}
\usepackage{tabularx}
\usepackage{array}
\usepackage{multicol}
\usepackage{multirow}
\usepackage{comment}

\theoremstyle{definition}
\newcolumntype{Y}{>{\centering\arraybackslash}X}

\newcommand{\x}{\mathbf{x}}

\newcommand{\lc}{\mathbf{c}}

\newtheorem{prop}{Proposition}
\newtheorem{cond}{Condition}
\usepackage{titlesec}
\title{TPLA: Tensor Parallel Latent Attention for Efficient Disaggregated Prefill \& Decode Inference}
\author{Xiaojuan Tang$^{1,3}$\thanks{Equal contribution.}, Fanxu Meng $^{1,3}$\footnotemark[1], Pingzhi Tang$^{1}$, Yuxuan Wang$^{1}$, Di Yin$^{3}$, Xing Sun$^{3}$, Muhan Zhang$^{1,2}$\thanks{Corresponding author: \texttt{muhan@pku.edu.cn}} \\
  $^{1}$Institute for Artificial Intelligence, Peking University \\
  $^{2}$State Key Laboratory of General Artificial Intelligence, BIGAI \\
  $^{3}$Tencent Youtu Lab, Shanghai, China\\
  \centering\href{https://github.com/fxmeng/TransMLA}{https://github.com/fxmeng/TransMLA}}

\begin{document}
\maketitle
\begin{abstract}
Multi-Head Latent Attention (MLA), introduced in DeepSeek-V2, compresses key–value states into a low-rank latent vector $\mathbf{c}^{\mathrm{KV}}$, caching only this vector to reduce memory. In tensor parallelism (TP), however, attention heads are computed across multiple devices, and each device must load the full $\mathbf{c}^{\mathrm{KV}}$, eroding the advantage of MLA over Grouped Query Attention (GQA).
We propose Tensor-Parallel Latent Attention (TPLA): a scheme that partitions both the latent representation and each head’s input dimension across devices, performs attention independently per shard, and then combines results with an all-reduce. TPLA preserves the benefits of a compressed KV cache while unlocking TP efficiency. Unlike Grouped Latent Attention (GLA), every head in TPLA still leverages the full latent representation, maintaining stronger representational capacity.
TPLA is drop-in compatible with models pre-trained using MLA: it supports MLA-style prefilling and enables efficient tensor-parallel decoding without retraining. Applying simple orthogonal transforms—e.g., the Hadamard transform or PCA—before TP slicing further mitigates cross-shard interference, yielding minimal accuracy degradation.
By reducing the per-device KV cache for DeepSeek-V3 and Kimi-K2, we achieve \(\mathbf{1.79\times}\) and \(\mathbf{1.93\times}\) speedups, respectively, at a 32K-token context length while maintaining performance on commonsense and LongBench benchmarks. TPLA can be implemented with FlashAttention-3, enabling practical end-to-end acceleration.
\end{abstract}
\begin{figure*}[th]
\centering
  \includegraphics[width=\linewidth]{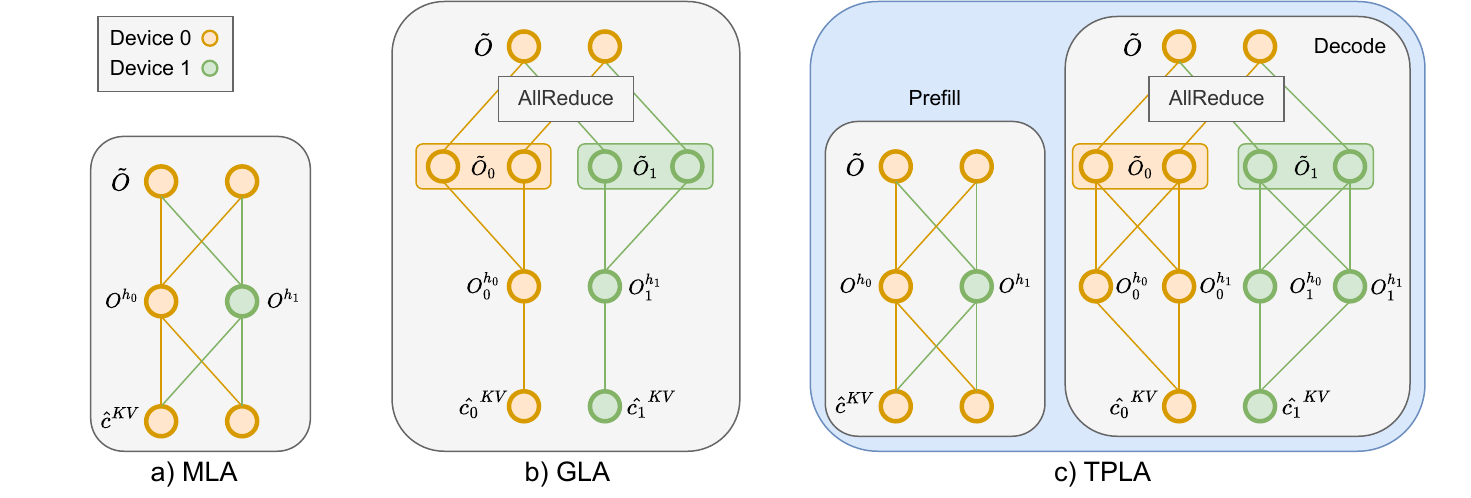}
  \caption{Comparison of MLA, GLA, and TPLA. In MLA, each device must load the entire KV cache. In GLA, each attention head only accesses the portion of the KV cache stored on its own device. In TPLA, the prefilling phase follows MLA for efficiency and accuracy, while during the decoding phase, attention heads are distributed across devices, each relying on the KV cache stored locally on its assigned device.}
\label{fig:tpla}
\end{figure*}
\section{Introduction}
\label{sec:intro}
Currently, large language models (LLMs) \cite{gpt4o, claude35sonnet, team2024gemini, radford2018improving, brown2020language} are typically memory-bound (limited by memory bandwidth) rather than compute-bound (limited by floating-point operations per second, FLOPs) during inference. To address this, KV cache compression \cite{chichih2024palu, chichih2025xkv, matanel2024transformers, zirui2024kivi} and tensor parallelism \cite{muru2025ladderresidual, hanbyul2025spd, itay2025tensorparallelism, qingyuan2024flash, shaden2022using} have emerged as two critical techniques for enabling efficient auto-regressive decoding in LLMs.
KV cache compression methods prune/merge/share/quantize intermediate key–value pairs to reduce memory overhead.
Tensor parallelism addresses memory and compute limitations by splitting large tensors—such as weight matrices—across multiple devices, enabling intra-layer parallel computation for models that cannot fit on a single GPU.
GQA \cite{ainslie2023gqa} inherently supports both KV cache compression and tensor parallelism by grouping query heads so that all heads within a group share a common set of key and value representations, which facilitates efficient distribution across multiple devices.
Both theoretical analyses and empirical results demonstrate that the representational capacity of GQA is inferior to that of MLA \cite{meng2025transmla,dsvii}. MLA introduces a pre-trained KV cache compression strategy that achieves an excellent trade-off between computational efficiency and model performance. However, when multiple attention heads are computed in parallel across multiple devices using tensor parallelism, MLA encounters a critical limitation: each device must load the full latent vector $c_{KV}$, undermining the memory savings that MLA offers over GQA. For example, in LLaMA-3-70B \cite{llama3}, the dimension of the KV cache per token is 2 $\times$ 8 $\times$ 128 = 2048, and under tensor parallelism with \text{TP} = 4, each device holds a partitioned KV cache of size 512. In contrast, Deepseek-V3 \cite{dsviii} has a fixed KV cache dimension of 64 + 512 = 576, which must be fully replicated on each device regardless of the parallelism degree. This results in a higher per-device KV cache memory footprint compared to GQA-based models under the same tensor parallel configuration.

GLA \cite{ted2025hardwareefficient} was proposed to address the tensor parallelism limitations of MLA by dividing the attention heads and latent representations into $g$ groups (typically $g$ = 2), such that each group of heads only loads its corresponding latent representation. However, this paper identifies two key limitations of GLA: (1) the reduction in KV cache size for single device comes at the cost of decreased representational capacity for each attention head; and (2) GLA requires training from scratch, which demands significant computational resources to validate its effectiveness.

To address these challenges, we propose Tensor Parallel Latent Attention (TPLA), a method that distributes the latent representations across multiple devices. Each attention head is split across devices, followed by an all-reduce operation on the output $o$. TPLA offers the following advantages:
1) Each attention head utilizes the full latent representation, preserving strong representational capacity;
2) Each device only loads a partition of the KV cache, improving inference speed under tensor parallelism;
3) TPLA can directly load pre-trained DeepSeek checkpoints, which incurs only a minor performance drop that is easily recovered;
4) We use reparameterized MLA for prefill and TPLA for decoding, reducing prefill latency while mitigating conversion-induced degradation.
4) TPLA can be viewed as a special case of GLA with more attention heads, making it compatible with FlashAttention-3.
\section{Related Works}
\paragraph{Reducing KV-Cache Memory}
Generative inference with large language models (LLMs) is often constrained by the memory footprint of the key--value (KV) cache, especially for long contexts. Several families of techniques have been explored to mitigate this burden: 
\textbf{token pruning/evicting} \cite{suyu2023model,xiabin2024dynamickv,xiaolin2025compresskv,yuhong2024snapkv,qichen2024lazyllm,zhenyu2023h2o} removes KV entries for low-importance tokens based on saliency or attention estimates;
\textbf{token merging} \cite{zheng2024model,xin2025zsmerge,jie2025efficient} aggregates nearby or similar tokens into a single surrogate KV representation to eliminate redundancy while retaining context; 
\textbf{cross-layer KV sharing/fusion} \cite{yifei2024kvsharer,shashank2024inferencefriendly,haoyi2024layercondensed,you2024a,william2024reducing} reuses one KV cache across adjacent layers to avoid per-layer storage; 
\textbf{low-rank KV compression} \cite{rongzhi2024lorc,bokai2024matryoshkakv,hao2024effectively,chichih2024palu} factorizes KV matrices into low-rank components (learned or SVD-based) to reduce dimensionality and memory; 
and \textbf{KV-cache quantization} \cite{coleman2024kvquant,zongwu2025million,shichen2024qaq,dingyu2025tailorkv} stores K/V tensors at reduced numeric precision (e.g., int8 or int4), cutting memory and bandwidth with modest accuracy cost. Although effective, these approaches inevitably discard or alter information in the KV cache and can degrade model performance. In contrast, \textbf{TPLA} leaves the KV contents intact: it reduces the amount of cache each device must hold so the model retains full information while alleviating memory pressure. As a result, TPLA tends to preserve accuracy better than compression-based methods.

\paragraph{Parallelism Strategies for Deployment}

Current LLMs scale to billions of parameters; to cope with the resulting memory and compute demands, engineers adopt distributed deployment to reduce wall-clock latency and time costs. 
\textbf{Data parallelism} \cite{dean2012large,sergeev2018horovod} partitions input data across the sample or batch dimension while replicating model parameters across devices. However, for very large models full replication becomes impractical; moreover, variable sequence lengths introduce load imbalance (``bubbles'') that waste compute resources.
\textbf{Pipeline parallelism} \cite{huang2019gpipe,narayanan2019pipedream} partitions the model into contiguous blocks of layers, each placed on a different device. Intermediate activations and gradients are communicated between stages to complete the forward and backward passes, reducing cross-node traffic. This staging overlaps computation across devices to increase throughput, but pipeline bubbles can still leave some devices idle.
\textbf{Tensor parallelism} \cite{shoeybi2019megatron,qifan2021an,zhengda2021maximizing} splits linear layers along their row or column dimensions, sharding tensors across devices and performing distributed matrix--matrix multiplication with collective communication. It archieves optimal performance on systems where GPUs are fully interconnected via NVLink. TPLA leverages the strengths of TP while addressing MLA’s inability to reduce the KV cache under TP.
Long sequences inflate the memory footprint of intermediate activations; \textbf{sequence parallelism} \cite{shenggui2021sequence} mitigates this by replicating the model across devices and splitting inputs along the sequence dimension so that each device processes only a subsequence.
\textbf{Prefill/Decode Separation} \cite{dynamo2025,zhao2025sandwich,zhong2024distserve} refines sequence parallelism for LLM inference: the prefilling phase is compute-intensive and thus compute-bound, whereas the decoding phase has low per-token compute but frequent memory accesses and is memory-bandwidth-bound. To match these characteristics, different machine counts and architectures are used across the two phases to improve latency and throughput. In TPLA, we further employ different model structures across phases---MLA during prefill to preserve accuracy while reducing computation (improving latency), and TPLA during decoding to reduce memory traffic and increase throughput.
\section{Preliminary}
\label{sec:preliminary}
\subsection{Multi-Head Latent Attention}
\label{sec:MLA}
MLA is designed to reduce memory bandwidth overhead by compressing the Key-Value (KV) cache. Specifically, the multi-head keys and values are compressed into a single low-rank latent representation of dimension $4d_h$, denoted as $\lc^{\mathrm{KV}}$. Instead of reconstructing full-size keys and values from this latent representation, MLA adopts a more efficient decoding strategy. By isolating the Rotary Position Embedding (RoPE) operation, the up-projection matrix can be absorbed into the query activations, yielding $Q$. Similarly, the value projection is absorbed into the output projection matrix, resulting in $W^{VO}$ (See Section~\ref{sec:absorb}). This allows for direct attention computation between $Q$ and the normalized latent cache $\hat{\lc}^{\mathrm{KV}}$, followed by a projection through $W^{VO}$ to produce the final output $\tilde{O}$. For simplicity in this initial description, we omit the RoPE components. The core computation is as follows:
\begin{align}
& c^{\mathrm{KV}}\in\mathbb{R}^{B\times L\times4d_h},\quad\;\;
\hat{\lc}^{\mathrm{KV}} = \mathrm{RMSNorm}(\lc^{\mathrm{KV}})\in\mathbb{R}^{B\times L\times4d_h},\nonumber\\
&Q\in\mathbb{R}^{B\times 1\times h_q\times4d_h},\quad
W^{VO}\in\mathbb{R}^{\bigl(h_q\cdot 4d_h \bigr)\times D},\nonumber\\
&O =\mathrm{softmax}\;\!\Bigl(\frac{Q\,(\hat{\lc}^{\mathrm{KV}})^\top}{\sqrt{d_h}}\Bigr)\,\hat{\lc}^{\mathrm{KV}}\in\mathbb{R}^{B\times 1\times h_q\times4d_h},\label{eq:mla_softmax}\\
&\tilde{O} = O\,W^{VO}\in\mathbb{R}^{B\times 1\times D}.
\label{eq:mla}
\end{align}

\subsection{Grouped Latent Attention}
\label{sec:GLA}

During tensor-parallel decoding, MLA replicates its single latent head on every device, resulting in high KV-cache memory load across all devices. GLA avoids this replication by partitioning the latent KV cache itself.
Consider a two-way tensor-parallel configuration. The latent KV cache is divided into two shards, $\lc_{\bm{0}}^{\mathrm{KV}}$ and $\lc_{\bm{1}}^{\mathrm{KV}}$, each assigned to one GPU. Simultaneously, attention heads $h_q$ are grouped such that the absorbed query projection matrix $Q$ and output projection matrix $W^{VO}$ are partitioned along both the head dimension ($h_q$) and the feature dimension ($4d_h$), yielding four groups. Thus, GPU 0 operates on $(\lc_{\bm{0}}^{\mathrm{KV}}, Q_{\bm{0,0}}, W_{\bm{0,0}}^{VO})$, while GPU 1 operates on $(\lc_{\bm{1}}^{\mathrm{KV}}, Q_{\bm{1,1}}, W_{\bm{1,1}}^{VO})$. Each GPU independently computes its local attention output, denoted $\tilde{O}_{\bm{0}}$ and $\tilde{O}_{\bm{1}}$, respectively. The final output is obtained via an $\mathrm{AllReduce}$ operation that sums the local outputs across devices.
\begin{align}
&\lc_{\bm{0}}^{\mathrm{KV}},\lc_{\bm{1}}^{\mathrm{KV}}\in\mathbb{R}^{B\times L\times2d_h},
\begin{cases}
\hat{\lc_{\bm{0}}}^{\mathrm{KV}} = \mathrm{RMSNorm} (\lc_{\bm{0}}^{\mathrm{KV}})\in\mathbb{R}^{B\times L\times2d_h},\nonumber\\
\hat{\lc_{\bm{1}}}^{\mathrm{KV}} = \mathrm{RMSNorm}(\lc_{\bm{1}}^{\mathrm{KV}})\in\mathbb{R}^{B\times L\times2d_h},
\end{cases}
\nonumber\\
&Q_{\bm{i,j}\in\{0,1\}}\in\mathbb{R}^{B\times 1\times \tfrac{h_q}{2}\times2d_h},\quad
\begin{pmatrix}
    Q_{\bm{0,0} },Q_{\bm{0,1}}\nonumber\\
    Q_{\bm{1,0}},Q_{\bm{1,1}}
\end{pmatrix} = Q,\nonumber\\
&W_{\bm{i,j}\in \{0,1\}}^{VO} \in \mathbb{R}^{\left(\tfrac{h_q}{2}\cdot 2d_h \right)\times D}, \quad\
W^{VO} = \begin{pmatrix}
    W_{\bm{0,0}}^{VO},\, W_{\bm{0,1}}^{VO} \,\nonumber\\
    W_{\bm{1,0}}^{VO}, \,W_{\bm{1,1}}^{VO}
\end{pmatrix}, \notag
\end{align}

\begin{align}
&O_{\bm{0}}=\mathrm{softmax}\;\!\Bigl(\frac{Q_{\bm{0,0}}\,(\hat{\lc}_{\bm{0}}^{\mathrm{KV}})^\top}{\sqrt{d_h}}\Bigr)\,\hat{\lc}_{\bm{0}}^{\mathrm{KV}}\in\mathbb{R}^{B\times 1\times \tfrac{h_q}{2}\times2d_h},\nonumber\\
&O_{\bm{1}} =\mathrm{softmax}\;\!\Bigl(\frac{Q_{\bm{1,1}}\,(\hat{\lc}_{\bm{1}}^{\mathrm{KV}})^\top}{\sqrt{d_h}}\Bigr)\,\hat{\lc}_{\bm{1}}^{\mathrm{KV}}\in\mathbb{R}^{B\times 1\times \tfrac{h_q}{2}\times2d_h},\nonumber \\
    &\tilde{O_{\bm{0}}} = O_{\bm{0}}\,W_{\bm{0,0}}^{VO}\in\mathbb{R}^{B\times 1\times D},\;\tilde{O_{\bm{1}}} = O_{\bm{1}}\,W_{\bm{1,1}}^{VO}\in\mathbb{R}^{B\times 1\times D},\nonumber\\
&O = \mathrm{AllReduce}\;\!\Bigl(\tilde{O_{\bm{0}}} + \tilde{O_{\bm{1}}}\Bigr)\in\mathbb{R}^{B\times 1\times D}.
\end{align}
\subsection{Matrix Absorption}
\label{sec:absorb}
Considering we apply orthogonal transformations $U$ to reparameterize weight matrices, which involves matrix absorption. To make this process intuitive, we here present the complete calculation pipeline of MLA and show how the absorbed matrices from Section~\ref{sec:MLA} are derived.

As stated in Section~\ref{sec:MLA}, MLA saves KV cache by multiplying the low-rank compress matrix $W^{DKV} \in \mathbb{R}^{D \times 4d_h}$ with the input sequence $X \in \mathbb{R}^{B \times L \times D}$ to obtain low-rank latent features $\mathbf{c}^{\mathrm{KV}}$.
Then, it uses the matrices $W^{UK}, W^{UV} \in \mathbb{R}^{4d_h \times (h_q \cdot d_h)}$ to derive the full-heads key $\mathbf{k}$ and value $\mathbf{v} $.
Additionally, MLA also can decompose $W^{Q} \in \mathbb{R}^{D \times (h_q \cdot d_h)}$ to $W^{DQ} \in \mathbb{R}^{D \times r_q} $ and $W^{UQ} \in \mathbb{R}^{r_q \times (h_q \cdot d_h)}$, which reduces the activation memory during training.
For positional embedding, MLA uses a decoupled RoPE strategy that uses additional multi-head queries $\mathbf{q}^\mathrm{PE}$ and a shared key $\mathbf{k}^\mathrm{PE}$, generated by $W^{QR}  \in \mathbb{R}^{r_q \times ( h_q \cdot d_r )} $ and $W^{KR} \in \mathbb{R}^{ D \times d_r }$, to carry the rotary positional embeddings. The final attention output $\tilde{O}$ is computed by separately combining the non-positional part ($\mathbf{q}\,\mathbf{k}^\top$) and positional part ($\mathbf{q}^\mathrm{PE}\,{(\mathbf{k}^{\mathrm{PE}}})^\top$), followed by projection with $W^O \in \mathbb{R}^{(h_q \cdot d_h) \times D}$.
\begin{align}
    &\mathbf{c}^{\mathrm{KV}} = X W^{DKV}, \quad \mathbf{c}^{\mathrm{Q}} = X W^{DQ} , \quad \hat{\mathbf{c}}^{\mathrm{KV}}=\mathrm{RMSNorm}(\mathbf{c}^{\mathrm{KV}}), \notag \\
    & \mathbf{q}  = \mathbf{c}^{Q}\,W^{UQ}  ,\quad\quad \mathbf{k} = \hat{\mathbf{c}}^{\mathrm{KV}}\,W^{UK} ,\ \; \mathbf{v}  = W^{UV} \hat{\mathbf{c}}^{\mathrm{KV}}, \quad  \notag \\
    &  \mathbf{q}^{\mathrm{PE}} = \mathrm{RoPE}(\mathbf{c}^{Q}\,W^{QR}), \quad\quad\quad\quad\;\; \mathbf{k}^{\mathrm{PE}} = \mathrm{RoPE}(X\,W^{KR}),  \notag\\ &O=\mathrm{softmax}\Bigl(\frac{\mathbf{q}\,\mathbf{k}^\top + \mathbf{q}^\mathrm{PE}\,({\mathbf{k}^{\mathrm{PE}}})^\top }{\sqrt{d_h + d_r }}\Bigr)\, \mathbf{v} , \quad \tilde{O} = O\,W^O.  \label{eq:isolate_rope}
\end{align}

In Equation~\ref{eq:isolate_rope}, the RoPE component is explicitly isolated, allowing us to restructure the attention computation using associativity of matrix multiplication. For clarity, we can temporarily omit positional encoding components and the scaling factor.
\begin{align}
    O\,W^{O}&=\mathrm{softmax}(\mathbf{q}\,\mathbf{k}^\top)\, \mathbf{v}\,W^O \notag\\
    &=\mathrm{softmax}(\mathbf{q}\,(\hat{\mathbf{c}}^{\mathrm{KV}} W^{UK})^\top)\, \hat{\mathbf{c}}^{\mathrm{KV}} W^{UV}\,W^O \label{eq:absorb_rmsnorm} \notag\\
    &=\mathrm{softmax}(Q({\hat{\mathbf{c}}^{\mathrm{KV}}})^\top)\, \hat{\mathbf{c}}^{\mathrm{KV}} W^{VO}.
\end{align}
Here, the matrix $W^{UK}$ can be absorbed into $\mathbf{q}$ to derive $Q$ in Equation~\ref{eq:mla}. Similarly, the matrix $W^{UK}$ can be absorbed into $W^O$. In practice, however, $W^{UV}$ is typically not absorbed into $W^O$ to avoid generating an impractically large matrix.

\newpage
\section{Tensor Parallel Latent Attention (TPLA)}
\label{sec:tpla}
Motivated by the hardware efficiency of GLA, we retain its core principle of distributing latent KV across GPUs to mitigate memory wastage and communication overload. However, directly translating an existing MLA-based model to GLA incurs a significant performance penalty, as shown in Figure~\ref{fig:mla_gla_tpla}. This degradation stems from a key limitation in standard GLA: the latent vector within each group only accesses half of the query heads, restricting the model's expressive power and leading to suboptimal performance. Moreover, training a new GLA model from scratch requires a substantial cost. To address this, we further propose Tensor-Parallel Latent Attention (TPLA). Unlike standard GLA, TPLA partitions latent vectors into two groups while preserving full query heads visibility. Specifically,
\begin{align}
&\lc_{\bm{0}}^{\mathrm{KV}},\lc_{\bm{1}}^{\mathrm{KV}}\in\mathbb{R}^{B\times L\times2d_h},
\begin{cases}
    \hat{\lc_{\bm{0}}}^{\mathrm{KV}} = \mathrm{RMSNorm} (\lc_{\bm{0}}^{\mathrm{KV}})\in\mathbb{R}^{B\times L\times2d_h}, \nonumber\\
\hat{\lc_{\bm{1}}}^{\mathrm{KV}} = \mathrm{RMSNorm}(\lc_{\bm{1}}^{\mathrm{KV}})\in\mathbb{R}^{B\times L\times2d_h},
\end{cases}\nonumber\\
&Q_{\bm{0}},Q_{\bm{1}}\in\mathbb{R}^{B\times 1\times h_q\times2d_h},\quad\quad\quad
\begin{pmatrix}
    Q_{\bm{0}}, Q_{\bm{1}}
\end{pmatrix} = Q,\nonumber\\
&W_{\bm{0}}^{VO},W_{\bm{1}}^{VO}
  \in\mathbb{R}^{\bigl(h_q\cdot 2d_h \bigr)\times D},\quad\;
  \begin{pmatrix}
    W_{\bm{0}}^{VO}, W_{\bm{1}}^{VO}
\end{pmatrix} = W^{VO},\nonumber\\
&O_{\bm{0}}=\mathrm{softmax}\;\!\Bigl(\frac{Q_{\bm{0}}\,(\hat{\lc}_{\bm{0}}^{\mathrm{KV}})^\top}{\sqrt{d_h}}\Bigr)\,\hat{\lc}_{\bm{0}}^{\mathrm{KV}}\in\mathbb{R}^{B\times 1\times h_q\times2d_h},\nonumber\\
&O_{\bm{1}}=\mathrm{softmax}\;\!\Bigl(\frac{Q_{\bm{1}}\,(\hat{\lc}_{\bm{1}}^{\mathrm{KV}})^\top}{\sqrt{d_h}}\Bigr)\,\hat{\lc}_{\bm{1}}^{\mathrm{KV}}\in\mathbb{R}^{B\times 1\times h_q\times2d_h},\label{eq:tpla_softmax_one_device}\\
&\tilde{O_{\bm{0}}} = O_{\bm{0}}\,W_{\bm{0}}^{VO}\in\mathbb{R}^{B\times 1\times D},\quad\quad
\tilde{O_{\bm{1}}} = O_{\bm{1}}\,W_{\bm{1}}^{VO}\in\mathbb{R}^{B\times 1\times D},\nonumber\\
&O = \mathrm{AllReduce}\;\!\Bigl(\tilde{O_{\bm{0}}} + \tilde{O_{\bm{1}}}\Bigr)\in\mathbb{R}^{B\times 1\times D}.
\end{align}

This design ensures each latent vector attends to all query heads, mitigating the most performance loss. The residual performance loss now stems exclusively from tensor-parallel partitioning effects in RMSNorm and softmax operations. Through carefully designed mathematical reparameterization, TPLA can restore near-MLA performance. For illustration, we consider the case where the tensor-parallel degree of latent attention is 2, though the approach naturally scales to higher degrees.

\subsection{RMSNorm Slicing}
\label{sec:rms_slicing}
In MLA-like models, the ``kv\_a\_layernorm'' module normalizes input vectors using the Root Mean Square (RMS) value. Given an input vector $\x \in \mathbb{R}^d$ (e.g., $d=4d_h$), the RMSNorm is computed as:
\begin{align}
    \text{RMS}(\x) &= \sqrt{ \frac{1}{d} \sum_{i=1}^{d} x_i^2 + \epsilon } \nonumber\\
    &=\sqrt{\frac{1}{d} \|\x\|_2^2+\epsilon}, 
\end{align}
\begin{align}
    \text{RMSNorm}(\mathbf{\gamma},\x) &=  \frac{\x}{\text{RMS}(\x)}\odot \mathbf{\gamma}\nonumber\\
    &=\text{RMSNorm}(\mathbf{1}, \x)\odot \mathbf{\gamma},
\end{align}
where $\epsilon$ is a small constant for numerical stability; $\gamma \in \mathbb{R}^d$ is a learned scaling parameter and $\odot$ denotes element-wise multiplication. 

However, we face the following challenge when applying this to tensor-parallel processing of latent attention: When input latent vector $\x \in \mathbb{R}^d$ is split into two partitions, $\x^{(0)} \in \mathbb{R}^{d/2}$ and $\x^{(1)} \in \mathbb{R}^{d/2}$, across different devices, the RMS computation on each local device uses only half the original dimension ($d/2$), while the true normalization requires the full $\text{RMS}(\x)$ over dimension $d$.

To resolve this discrepancy, we introduce an orthogonal transformation $U \in \mathbb{R}^{d \times d}$ ($U\,U^\top = \mathbf{I}$) to reparamerize this module. Before introducing the conditions that this transformation $U$ need satisfy, we first establish that RMSNorm can, in principle, be realized in a mathematically equivalent form under any orthogonal transformation.
\begin{prop}
\label{prop:rmsnorm}
\begin{equation}
    \text{RMSNorm}(\mathbf{1},\lc) = \text{RMSNorm}(\mathbf{1},\lc\,U)\,U^\top
\end{equation}
\end{prop}

\begin{proof}
we first represent the RMSNorm process as matrix multiplication. Let $\mathbf{c} \in \mathbb{R}^{L \times d}$ be the input latent vector (for simplicity, we omit the batch size), we obtain: 
\begin{align}
    \text{RMSNorm}(\mathbf{\gamma}, \mathbf{c}) &=\text{RMSNorm}(\mathbf{1}, \mathbf{c})\, W_\gamma,\label{eq:rmsnorm}\\
    &= D_c \,\mathbf{c}\, W_\gamma,
\end{align}
where $D_c$ is a diagonal matrix of size $L \times L$ with the reciprocal of the RMS values on the diagonal and $W\gamma$ is also a diagonal matrix of size $d \times d$ with each learnable scaling parameter:
\begin{align}
    &D_c = \text{diag}\,\bigl(\frac{1}{\text{RMS}(c_1)},\frac{1}{\text{RMS}(c_2)}, \ldots, \frac{1}{\text{RMS}(c_L)}\,\bigr),\\
    &W_\gamma = \text{diag}\,\bigl( \gamma_1, \gamma_2, \ldots, \gamma_d \,\bigr).
\end{align}
Since the orthogonal transformation preserves the norm ($\|\mathbf{c}\,U\|_2^2 = \|\mathbf{c}\|_2^2$), we easily have $\text{RMS}(\mathbf{c})=\text{RMS}(\mathbf{c}\,U)$, i.e., $D_{cU} = D_c$. Thus, we can have:
\begin{equation}
\label{eq:rmsnorm_eq}
\text{RMSNorm}\,(\gamma,\mathbf{c}\,U\,)\,U^\top  = D_{c}\,\mathbf{c}\,U\,W_\gamma\, U^\top.
\end{equation}
Matrix multiplication does not satisfy the commutative property. Therefore, when and only when $W_\gamma = \mathbf{I}$, we can further prove:
\begin{align}
    \text{RMSNorm}(\mathbf{1},\mathbf{c}\,U)U^\top  &= D_{c}\,\mathbf{c}\,U\,\mathbf{I}\, U^\top \notag\\ &= D_{c}\,\mathbf{c} \notag\\&=  \text{RMSNorm}(\mathbf{1},\mathbf{c}).
\end{align}
\end{proof}
Give by Equation~\ref{eq:absorb_rmsnorm}, Equation~\ref{eq:rmsnorm} and Proposition~\ref{prop:rmsnorm}, we can absorb $W_\gamma$ into up-projection matrix $W^{UKV} = (W^{UK}, W^{UV})$ to achieve the $\gamma=\mathbf 1$, ensuring the orthogonal transformations $U$ to $\mathbf c$ with keeping the RMSNorm value no change. In addition, the $U^T$ can be further absorbed into $W^{UKV}$; $U$ can be absorbed into $W^{DKV}$, yielding the reparameterized weight matrix:
\begin{equation}
    W^{UKV}_{new} =  U^\top\,W_\gamma\,W^{UKV}, \quad W^{DKV}_{new} = W^{DKV}\,U.
\end{equation}
We have proved that any transformation $U$ can ensure the equivalence of RMSNorm. Now we will define some conditions that serve as the computational basis for $U$, deferring the specific calculation method to a later section.
\begin{cond}[RMSNorm Slicing Condition]
\label{cond:rmsnorm_slicing}
\begin{equation}
    \alpha \|(\mathbf{c}\,U)_{\mathbf 0}\|^2_2 \approx \beta \|(\mathbf{c}\,U)_{\mathbf 1}\|^2_2 \approx \|\mathbf{c}\,U\|^2_2 = \|\mathbf{c}\|^2_2.
\end{equation}
\end{cond}
Here, $\alpha$ and $\beta$ are fixed constants, invariant to changes in the input data distribution (How to calculate their specific values is detailed in Section~\ref{sec:reparameterization}). $(\mathbf{c}\,U)_{\mathbf 1}$ and $(\mathbf{c}\,U)_{\mathbf 2}$ are the two partitions of the transformed $\mathbf{c}$ split across devices. By satisfying this, the new RMS values computed from the two partitions are proportional to the global value, thereby providing an accurate approximation of the global RMSNorm:
\begin{align}
    \text{RMS}(\mathbf{c}) &=\sqrt{\frac{1}{d} \|\lc\|_2^2+\epsilon} \notag \\
    &\approx \sqrt{\frac{\alpha}{d} \|(\lc\,U)_{\bm{0}}\|_2^2+\epsilon} \notag \\
    &\approx \sqrt{\frac{\alpha}{{2}}} \,\text{RMS}\left((\mathbf{c}\,U)_\mathbf{0}\right) \approx \sqrt{\frac{\beta}{{2}}}\,\text{RMS}\left((\mathbf{c}\,U)_\mathbf{1}\right).
\end{align}
Thus, we can compute RMSNorm in a tensor-parallel manner while maintaining the integrity of the normalization process.





\subsection{Softmax Slicing}
\label{sec:softmax_slicing}

In common tensor-parallel techniques, matrices are typically split across devices to perform either row or column parallelism. For our TPLA attention score computation, row parallelism is employed, where the weight matrix A is split across devices according to its rows. To ensure a valid matrix multiplication, the input matrix X is correspondingly partitioned column-wise into $X_1$ and $X_2$, such that
\begin{equation}
    XA = \begin{pmatrix} X1 & X2 \end{pmatrix} \begin{pmatrix} A1 \\ A2 \end{pmatrix} = X1 \cdot A1 + X2 \cdot A2 = Y1 + Y2 = Y
\end{equation}
where $X_1$ and $A_1$ are computed on GPU 0 to produce $Y_1$, and $X_2$ and $A_2$ are computed on GPU 1 to produce $Y_2$. The final output $Y$ is then all-reduced by summing $Y_1$ and $Y_2$.

In the context of softmax computation (Equation~\ref{eq:isolate_rope}), TPLA partitions $\mathbf{c}^{\mathrm{KV}}$ and ensures that the computation of positional components remains unaffected. Specifically, the shard of the key positional embedding ${\mathbf{k}^{\mathrm{PE}}}$ must be replicated across devices so that the local positional values remain consistent with the global values. As for non-positional parts $(Q\,(\hat{\mathbf{c}}^{\mathrm{KV}})^\top)$, the latent vectors $\mathbf{c}$ are split into two devices, and each device performs only its local computation, i.e., GPU 0 computes $Q_{\bm{0}}\,(\hat{\lc}^{\mathrm{KV}}_{\bm 0})^\top$, while GPU 1 computes $Q_{\bm{1}}\,(\hat{{\lc}}^{\mathrm{KV}}_{\bm 1})^\top$. However, in most cases,
\begin{align}
     &\quad\,\mathrm{softmax}\bigl(Q\,(\hat{\mathbf{c}}^{\mathrm{KV}})^\top + \mathbf{q}^\mathrm{PE}\,({\mathbf{k}^{\mathrm{PE}}})^\top\bigr) \notag \\&
     = \mathrm{softmax}\bigl( Q_{\mathbf{0}}\,(\hat{\lc}_{\mathbf{0}}^{\mathrm{KV}})^\top + Q_{\mathbf{1}}\,(\hat{\lc}_{\mathbf{1}}^{\mathrm{KV}})^\top + \mathbf{q}^\mathrm{PE}\,({\mathbf{k}^{\mathrm{PE}}})^\top\bigr) \notag \\&
     \neq \mathrm{softmax}\bigl(Q_{\mathbf{0}}\,(\hat{\lc}_{\mathbf{0}}^{\mathrm{KV}})^\top + \mathbf{q}^\mathrm{PE}\,({\mathbf{k}^{\mathrm{PE}}})^\top\bigr)\quad \quad \quad \quad \quad \text{[GPU 0]}\notag \\&\neq \mathrm{softmax}\bigl(Q_{\mathbf{1}}\,(\hat{\lc}_{\mathbf{1}}^{\mathrm{KV}})^\top + \mathbf{q}^\mathrm{PE}\,({\mathbf{k}^{\mathrm{PE}}})^\top\bigr) \quad \quad \quad \quad \quad \text{[GPU 1]}\notag
\end{align}

Thus, the challenge of TPLA is how to ensure 
the global value of $Q\,\lc^{\mathrm{KV}})^\top$ can be approximated from local computations. We first show that applying any orthogonal transformation $U$ does not alter the equivalence of the original softmax output. Based on Equation~\ref{eq:absorb_rmsnorm}, we easily have:
\begin{align}
\label{eq:absorb_softmax}
    Q\,(\hat{\lc}^{\mathrm{KV}})^\top  =  QU\,(\hat{\lc}^{\mathrm{KV}}U)^\top &= \mathbf{q}\,(U^\top\,W^{UK})^\top (\hat{\mathbf{c}}^{\mathrm{KV}}\,U)^\top\notag\\
    &=Q^\prime(\hat{\mathbf{c}}^{\mathrm{KV}}\,U)^\top
\end{align}
Analogous to Section~\ref{sec:rms_slicing}, by absorbing $U$ into $W^{UKV}$ and $W^{DKV}$ (equivalently, into $Q$ to obtain $Q^\prime$), we can impose an orthogonal transformation $U$, which preserves the original softmax computation and must satisfy: 
\begin{cond}[Softmax Slicing Condition]
\label{cond:softmax_slicing}
\begin{equation}
    Q^\prime\,(\hat{\mathbf{c}}^{\mathrm{KV}}\,U)^\top \approx \mu Q^\prime_{\bm 0}\,(\hat{\lc}^{\mathrm{KV}}\,U)_{\bm{0}}^\top \approx  \nu Q^\prime_{\bm{1}}\,(\hat{\lc}^{\mathrm{KV}}\,U)_{\bm{1}}^\top
\end{equation}
\end{cond}
Accordingly, by determining the coefficients $\mu$ and $\nu$, each device can perform its local computation and scale by the factor, thereby approximating the global value.



\subsection{Reparameterization Methods}\label{sec:reparameterization}
From the derivation above, we need to find one orthogonal transformation matrix $U$ applied to projection weights, ensuring that the transformation satisfies Condition~\ref{cond:rmsnorm_slicing} and Condition~\ref{cond:softmax_slicing} — local computations can accurately approximate the global RMSNorm and softmax values. To achieve this, we explore two potential methods: Hadamard Matrix Transformation and Principal Component Analysis (PCA).

\subsubsection{Hadamard Matrix Transformation}

Hadamard matrix is a special orthogonal matrix where each entry is either +1 or -1. It operates by balancing the numbers, thereby reducing extreme numerical deviations and promoting a more uniform distribution of data. In practice, we typically use the function \texttt{scipy.linalg.hadamard(d)} to generate a Sylvester-type Hadamard matrix (also known as a Walsh-Hadamard matrix) $H_d \in \mathbb{R}^{d \times d}$, constructed using a deterministic recursive rule:
\begin{equation}
H_{2n} =
\begin{pmatrix}
H_{n} & H_{n} \\
H_{n} & -H_{n}
\end{pmatrix},
\quad H_1 = (1).
\end{equation}

To increase robustness, a random diagonal matrix $D$, with entries drawn from ${\pm 1}$ is multiplied with $H_d$, thereby breaking deterministic structure while preserving orthogonality. Since $H_d H_d^\top = d \cdot\mathbf{I}$, orthonormality is achieved by scaling $H_d$ by $\tfrac{1}{\sqrt{d}}$, ensuring that normalization values are preserved.

Take an illustrative example. Consider a 4-dimensional vector $\mathbf{c} = (100, 0, 0, 0)$ and the $4 \times 4$ Hadamard matrix $H_4$. The transformed vector $\mathbf{c}^\prime = \mathbf{c} H_4$ is:
\begin{equation}
    \mathbf{c}^\prime = (100,0,0,0) \begin{pmatrix}
    \frac{1}{2} & \frac{1}{2} & \frac{1}{2} & \frac{1}{2} \\ \frac{1}{2} & -\frac{1}{2} & \frac{1}{2} & -\frac{1}{2} \\ \frac{1}{2} & \frac{1}{2} & -\frac{1}{2}& -\frac{1}{2} \\ \frac{1}{2} & -\frac{1}{2} & -\frac{1}{2} & \frac{1}{2}
    \end{pmatrix}=(50, 50, 50, 50).
\end{equation}

When applied to an input vector $\mathbf{c}$: $\mathbf{c}^\prime= \mathbf{c}\,H_d, \,d=4d_h$, we obtain: $\frac{\|(\mathbf{c}H_d)_1\|_2^2}{d/2} \approx  \frac{\|(\mathbf{c}H_d)_2\|_2^2}{d/2} \approx \frac{\|\mathbf{c}H_d\|_2^2}{d} = \frac{\|\mathbf{c}\|_2^2}{d}$, satisfies our key Condition~\ref{cond:rmsnorm_slicing} and easily determine $\alpha=2$. This uniformity minimizes approximation error in tensor-parallel RMSNorm, validated experimentally in Figure~\ref{fig:mla_gla_tpla}. 

However, satisfying Condition~\ref{cond:softmax_slicing} is more challenging. While the magnitudes of the Hadamard transformed vector elements are balanced, due to the presence of both positive and negative signs, this transformation does not guarantee that the multiplication of the two parts will be approximately clear.
To illustrate this, consider the following example: let $Q = (100, 0, 0, 0)$ and $\mathbf{c} = (0, 0, 80, 0)$. After applying the Hadamard transformation, we have:
\begin{align}
    &Q^\prime = Q H_4 = (50, 50, 50, 50), \quad \notag\\ &\mathbf{c}^\prime = \mathbf{c} H_4 = (40, 40, -40, -40). \notag
\end{align}
The element-wise product is:
$ Q \cdot \mathbf{c}^\prime = (200, 200, -200, -200).$ When this product is split into two parts, we get: $400 \neq -400 \neq 0$. This demonstrates that a standard Hadamard transformation cannot ensure Condition~\ref{cond:softmax_slicing}. A potential direction to address this issue is to search for an optimized Hadamard matrix via dimension permutations that minimizes the discrepancy between partitions. We leave the investigation of such optimized transformations for future work.

\subsubsection{Principal Component Analysis (PCA)}
\label{sec:PCA}
PCA is a widely used technique in statistics and machine learning for dimensionality reduction, feature extraction, etc. It transforms a dataset into a new coordinate system such that the greatest variances of the data are captured along the new axes (principal components). Each subsequent component is orthogonal to (i.e., uncorrelated with) the preceding ones.
In our context, we leverage this property to project data onto orthogonal dimensions, with the eigenvalues indicating the variance captured along each eigenvector and thus reflecting the statistical importance of each dimension. Moreover, for mean-centered features, the variance is equivalent to mean of the squared values, closely related to squared RMS value.

To implement this, we process a calibration dataset (e.g., Wikitext-2) to collect the KV latent cache (excluding position features) represented by $F \in \mathbb{R} ^{( B \cdot L) \times d}$. We then compute the eigenvectors $U$ and eigenvalues $\Lambda$ by performing eigenvalue decomposition on the covariance matrix $\Sigma_F = U\Lambda U^\top$. 

Based on Condition~\ref{cond:rmsnorm_slicing}, we define $\alpha$ as proportion of variance captured by the first $d/2$ principal components. Similarly, $\beta$ represents the proportion of variance captured by the remaining components. These ratios are as follows:
\begin{equation}
    \alpha = \frac{\sum_{i=1}^{d/2}{\lambda_i}}{\sum_{i=1}^{d}{\lambda_i}},\quad \beta = \frac{\sum_{i=d/2}^{d}{\lambda_i}}{\sum_{i=1}^{d}{\lambda_i}}.
\end{equation}
For Condition~\ref{cond:softmax_slicing}, the metrics $\mu$ and $\nu$ are defined in the same manner, making them equivalent to $\alpha$ and $\beta$, respectively.


\subsection{TPLA as a Special Case of GLA}\label{sec:tpla_as_gla}
Tensor parallelism in GLA employs a two-dimensional sharding scheme, splitting both head axis $h_q$ and the latent dimension axis $4d_h$ across devices. For a query tensor $Q \in \mathbb{R}^{B \times L \times h_q \times 4d_h}$, this partitioning yields four logical sub-tensors:
\begin{equation*}
    Q = \begin{pmatrix} Q_{\boldsymbol{0,0}} & Q_{\boldsymbol{0,1}} \\ Q_{\boldsymbol{1,0}} & Q_{\boldsymbol{1,1}} \end{pmatrix}, \quad
    \text{where} \quad Q_{\boldsymbol{i,j}} \in \mathbb{R}^{B \times L \times \tfrac{h_q}{2} \times 2d_h}.
\end{equation*}
In standard GLA, they are distributed with only two devices. Thus, only two diagonal blocks can be materialized locally—one per device—without additional communication:
\begin{equation*}
    \begin{cases}
    \text{Device 0: } Q_{\bm{0,0}} \in \mathbb{R}^{B \times L \times\frac{h_q}{2} \times 2d_h}, \\
    \text{Device 1: } Q_{\bm{1,1}} \in \mathbb{R}^{B \times L \times\frac{h_q}{2}\times 2d_h}.
    \end{cases}
\end{equation*}
 Each latent slice (of width $2d_h$) is paired with only half of the query heads ($h_q/2$) and thus unable to access the off-diagonal head slices $Q_{\bm{1,0}}$ and $Q_{\bm{0,1}}$. In effect, these parts do not contribute to the computation, resulting in significant performance degradation.

In contrast, TPLA overcomes this limitation by enabling each partitioned latent vector to attend to all query heads. It achieves this by reformulating the computation to be algebraically equivalent to a GLA system with double the number of heads. Concretely, define a conceptual query tensor $Q^\prime$ that duplicates the original query along the head dimension:
\begin{equation*}
    Q' = \begin{pmatrix}
    Q_{\bm{0,0}} & Q_{\bm{0,1}} \\ Q_{\bm{1,0}} & Q_{\bm{1,1}} \\
    Q_{\bm{0,0}} & Q_{\bm{0,1}} \\ Q_{\bm{1,0}} & Q_{\bm{1,1}}
    \end{pmatrix}
    \;\in\; \mathbb{R}^{B \times L \times (2h_q) \times (4d_h)},
\end{equation*}
For $h_q$ original heads, TPLA's duplication additional creates $h_q$ heads. This is algebraically equivalent to a GLA system with $2h_q$ heads $4d_h$ latent dimension. Thus, we can perfectly follow the same as TPLA sharding way. When split into two device, we have:
\begin{equation*}
    \begin{cases}
    \text{Device 0: } \begin{bmatrix} Q_{\bm{0,0}} & Q_{\bm{0,1}} \end{bmatrix} \in \mathbb{R}^{B \times L \times h_q \times 4d_h} \\
    \text{Device 1: } \begin{bmatrix} Q_{\bm{1,0}} & Q_{\bm{1,1}} \end{bmatrix} \in \mathbb{R}^{B \times L \times h_q \times 4d_h}
    \end{cases}
\end{equation*}

Generalizing to $k$ devices, let $g$ denote the TPLA replication factor (number of latent-cache-slice groups). TPLA divides $k$ devices into $r$ group size of size $k/g$. Each group holds a disjoint slice of the latent axis of width $4d_h/g$ and replicates the complete set of head parameters. Within each group, the head axis is sharded across the $k/g$ devices, Consequently, each device processes $\frac{h_q}{k/g}$ heads, and $\frac{4d_h}{g} $ \text{latent features}.  For $g=2,\,k=2$, this recovers the two-device case above, where each device receives $h_q$ heads and $2d_h$ latent width. The computational complexity arising from parameter replication is analyzed in Section~\ref{sec:pd_sep}.
 
In summary, because TPLA preserves GLA's sharding pattern—differing only by a constant-factor replication of head parameters—state-of-the-art attention optimizations (e.g., FlashAttention-3) can be applied to TPLA without substantial changes to the underlying framework. 


\subsection{Prefill-Decode Separation}\label{sec:pd_sep}
Large language model inference is usually into two phases with distinct performance characteristics: \textit{prefill} and \textit{decode}. The prefill phase processes the entire prompt in a single, parallel pass to compute the initial Key-Value (KV) cache. This large-batch computation is fundamentally compute-bound. The subsequent decode phase autoregressively generates one token at a time. Each generation step requires reading the entire, growing KV cache from high-bandwidth memory (HBM) to on-chip SRAM. As the context length increases, this large data transfer becomes the primary bottleneck, making the decode phase memory-bound.

Our proposed technique, TPLA, addresses this challenge by reducing the KV cache size on each device. This reduction effectively alleviates the memory bandwidth bottleneck at the cost of a minor increase in computation. In essence, TPLA shifts the decode phase from being memory-bound towards being more compute-bound. A detailed analysis is as below.

\paragraph{Complexity Analysis of TPLA}
To maximize the degree of tensor parallelism and enable acceleration methods compatible with GLA, TPLA requires the replication of head-specific parameters across latent attention groups, as discussed in Section~\ref{sec:tpla_as_gla}. Specifically, let's analyze the case with a tensor parallelism (TP) degree of 2. We consider a hidden state $X \in \mathbb{R}^{L_q \times D}$ for a single-batch inference using the MLA-absorbing strategy. The dominant cost lies in the attention computation. For a KV cache of length $S_\mathrm{{kv}}$, the complexity of the TPLA attention module (Equation~\ref{eq:tpla_softmax_one_device}) is approximately $\mathcal{O}(L_q \times S_{KV} \times h_q \times 2d_h \times 2 )$. In comparsion, for MLA (Equation~\ref{eq:mla_softmax}), with TP=2, the heads are split into two groups of $\frac{h_q}{2}$, leading to a complexity of $\mathcal{O}(L_q \times S_{KV} \times \frac{h_q}{2} \times 4d_h \times 2 )$. These two complexities are arithmetically equivalent. (Strictly speaking, positional components introduce additional overhead, since TPLA doubles the number of heads without reducing the RoPE dimension, but this effect is relatively minor.) Similarly, the $\tilde{O}$ computations are also equivalent. 

Beyond the main attention computation, TPLA modifies other calculations: the computation of $\mathbf{c}^{\mathrm{KV}}$ in TPLA is distributed across two devices, reducing complexity by $\mathcal{O}(L_q \times D \times 2d_h \times 2)$, while the computation of $Q$ increases by $\mathcal{O}(L_q \times D \times 2h_q \times d_h)$. However, as context length grows, the overall cost is increasingly dominated by the self-attention module (see Figure~\ref{fig:prefill_latency}).

To mitigate the additional computational overhead of TPLA, we also strategically decouple the attention mechanisms: retaining standard MLA during compute-intensive prefilling to minimize computation and reduce loss caused by converting MLA to TPLA, while activating TPLA exclusively during memory-bound decoding to minimize KV cache footprint. This hybrid approach thereby further optimizes performance by matching each phase to its most suitable mechanism.

\paragraph{Discussion}
The above analysis focuses on MLA-absorbing computation. However, when constructing the full-size KV cache and performing attention, TPLA doubles the number of heads while maintaining the same head dimension $d_h$. As a result, the overall computational load increases, making training TPLA from scratch potentially costly. Designing effective and efficient training strategies for TPLA remains an open problem, which we leave for future work. Nevertheless, one practical pathway is to adopt the existing MLA design during training and then convert it to TPLA with only minimal loss. This approach allows us to retain training efficiency while still benefiting from TPLA's advantages in inference scenarios.


\begin{table*}[t]
\centering
\caption{WikiText-2 Perplexity and Commonsense reasoning performance when converting the MLA to TPLA. The six benchmarks include MMLU, ARC (easy and challenge), PIQA, HellaSwag, OpenBookQA (OBQA), and Winogrande(WG).}
\begin{tabularx}{0.99\linewidth}{lYYYYYYYYY}
\toprule
\textbf{Model} & \textbf{PPL$\downarrow$} & \textbf{Avg.$\uparrow$} & \textbf{MMLU} & \textbf{ARC} & \textbf{PIQA} & \textbf{HellaSwag} & \textbf{WG} & \textbf{OBQA} \\
\midrule
\rowcolor{gray!10}DeepSeek-V2-Lite &6.31& 61.75 &43.19 &60.39 &80.20 &74.46 &65.43 &45.80   \\
\textit{- GLA}  &2212. & 33.77&25.32 &26.77 &51.47 &25.65&49.88&23.60 \\
\textit{- TPLA}  &7.24&54.33 &37.67 &51.50 &75.46 &63.56  &59.19 &38.60\\
\textit{- TPLA (align) }  &6.51&61.52&42.72&62.58&79.82&73.32&65.90&44.80 \\
\textit{- TPLA (pd sep.)} &6.31 &61.44 &43.19 &60.14 &80.09 &74.41 &65.59 &45.20 \\
\midrule
\rowcolor{gray!10}DeepSeek-V2  &3.89& 68.32&51.91&69.09&83.13&82.17&74.03&49.60\\
\textit{- TPLA}  &4.72&63.40&47.19&65.04&80.69&75.46&66.61&45.40 \\
\midrule
\rowcolor{gray!10}DeepSeek-V3  &3.24&72.10  &60.85&77.16 &85.58 &85.41&75.22&48.40 \\
\textit{- TPLA}  &4.02&68.00&54.88&75.25&82.70&80.69&69.46&45.00 \\
\midrule
\rowcolor{gray!10}Kimi-K2-Base   &1.91 &73.52 &63.20 &78.75 &85.47 &87.55 &75.93 &50.20\\
\textit{- TPLA} &2.44 &70.49 &57.64 &76.00 &83.79 &83.53 &72.38 &49.60\\
\midrule
\rowcolor{gray!10}LLaMA-2-7B  &5.47& 59.85 & 41.43 & 59.24 & 78.40 & 73.29 & 64.96 & 41.80 \\
\textit{- TransMLA} &5.88&58.95&40.38 &57.64&78.18&70.59 & 62.90 & 44.00 \\
\textit{- TPLA }&6.74& 54.68  &36.12  &53.21 &74.81 &64.52 &59.04 &40.40\\
\bottomrule
\end{tabularx}
\label{tab:cs}
\end{table*}

\section{Experiment}
\label{sec:exper}

An advantage of TPLA over GLA \cite{ted2025hardwareefficient} is that TPLA can be applied without training a model from scratch. It allows direct loading of models originally trained with MLA (e.g., the DeepSeek series \cite{dsvi,dsvii,dsviii}, Kimi-k2 \cite{team2025kimi}, TransMLA \cite{meng2025transmla}), and—through our proposed reparameterization method and Prefill/Decode Separation technology—mitigates performance degradation caused by changes in the attention mechanism. 

\subsection{Performance on Commonsense Tasks}
\label{sec:experiment_cs}
In this section, we evaluate TPLA by directly loading MLA checkpoints \emph{without any additional training} on short-text commonsense tasks. Performance is measured with the LightEval framework on MMLU \cite{MMLU}, ARC (Easy/Challenge) \cite{ARC}, PIQA \cite{PIQA}, HellaSwag \cite{HS}, OpenBookQA (OBQA) \cite{OBQA}, and WinoGrande (WG) \cite{WG}. Results are reported in Table~\ref{tab:cs}.
For \textbf{GLA}, following the procedure in Section~\ref{sec:GLA}, we partition the attention heads into two groups, assigning each group half of the latent dimension. As shown in Table~\ref{tab:cs}, discarding half of each head’s KV cache causes severe performance degradation—for example, WikiText-2 perplexity (ppl) increases from 6.31 with MLA to 2212 with GLA—whereas \textbf{TPLA}, which allows each attention head to use the full latent dimension across different devices, maintains a ppl of 7.24. This comparison indicates that TPLA preserves MLA’s representational capacity while reducing the per-device KV-cache footprint. We therefore expect that pretraining TPLA from scratch would outperform GLA.
For \textbf{TPLA}, we first use WikiText-2 \cite{merity2016pointer} as a calibration set and, following Sections~\ref{sec:rms_slicing} and~\ref{sec:softmax_slicing}, slice the MLA components (the $KV_a$ RMSNorm and the softmax) to obtain TPLA weights. As shown in Table~\ref{tab:cs}, this requires no fine-tuning and yields only minor accuracy degradation. The reparameterization method used here is the PCA-based approach described in Section~\ref{sec:PCA}.
For \textbf{TPLA (align)}, we use the SmolLM-Corpus \cite{benallal2024smollmcorpus} for lightweight alignment. First, we match the layer-wise input/output features of TPLA to those of the original MLA model using 256 random samples of length 2{,}048 for 10 epochs, minimizing MSE with the Muon optimizer (initial learning rate $1\mathrm{e}{-6}$). Next, we align the end-to-end model outputs using 100M tokens, following the TransMLA setting (batch size $=32$, learning rate $=2\mathrm{e}{-5}$, warmup ratio $=0.03$, cosine scheduler, max sequence length $=4096$). Experiments are conducted on a node with $8\times$ GPUs (96\,GB per GPU, $\sim$148 FP16 TFLOPS each). This small amount of alignment data is sufficient to recover the performance of the converted model.
For \textbf{TPLA (PD-sep.)}, we use MLA in the prefilling stage with the same reparameterization but \emph{without} slicing the RMSNorm or softmax; prefilling thus behaves identically to the original MLA, and the KV cache can be partially reused by TPLA during decoding. By avoiding slicing for most tokens, this prefill–decode separation achieves performance close to the original model \emph{without any training}.
For \textbf{LLaMA-2-7B}, we first apply TransMLA \cite{meng2025transmla} to convert MHA/GQA to MLA (64 RoPE dimensions and 512 NoPE dimensions—corresponding to a pruning ratio of 92.97\%.) and then fine-tune to recover performance. We subsequently convert the MLA checkpoint released by TransMLA directly into TPLA. With TransMLA as a bridge, TPLA can be applied to pretrained models that originally use MLA, GQA, or MHA.

These experiments demonstrate that converting MLA-based models to TPLA can effectively preserve performance. Given the benefits of TPLA for tensor parallelism, this presents a promising approach for efficient model deployment and acceleration.

\begin{table*}[t]
\centering
\caption{Longbench performance when converting the MLA to TPLA. }
\begin{tabularx}{\linewidth}{lYYYYYYY}
\toprule
\textbf{Model} & \textbf{Avg.} & \textbf{MultiQA} & \textbf{SingleQA} & \textbf{Summarize} & \textbf{Few-Shot} & \textbf{Code} & \textbf{Synthetic} \\
\midrule
\rowcolor{gray!10}DeepSeek-V2-Lite   &28.90 &12.43 &20.04 &16.74 &62.59 &57.86 &3.77 \\
\textit{- TPLA }  &10.98 &6.96 &9.20 &6.91 &25.29 &14.41 &3.11\\
\textit{- TPLA (align)}  &22.60 &10.97 &14.67 &16.59 &58.03 &31.59 &3.77\\
\textit{- TPLA (pd sep.)}  &24.44 &13.95 &15.07 &8.62 &59.10 &46.67 &3.23\\
\midrule
\rowcolor{gray!10}DeepSeek-V3   &58.19 &55.37 &51.65 &23.97 &69.42 &80.09 &68.63 \\
\textit{- TPLA } &44.52 &35.02 &38.53 &12.55 &53.01 &61.20 &66.83 \\
\textit{- TPLA (pd sep.) } &56.04 &53.01 &50.17 &21.39 &67.22 &75.97 &68.47\\
\bottomrule
\end{tabularx}
\label{tab:longbench}
\end{table*}    

\subsection{Performance on Longbench Tasks}
\label{sec:experiment_longbench}
As context length grows, memory traffic increases and the KV-cache size becomes a primary driver of latency and throughput. To assess how TPLA converted from MLA behaves on long inputs, we evaluate on \emph{LongBench} \cite{bai2023longbench}, a bilingual (English/Chinese), multi-task benchmark for long-context understanding that comprises 21 tasks across six categories (e.g., question answering, summarization, and few-shot learning). Because long-text inference is slower, we report results only for DeepSeek-V2-Lite and DeepSeek-V3. Due to GPU memory constraints, the maximum input context length is set to 31,500 tokens for DeepSeek-V2-Lite and 127,500 tokens for DeepSeek-V3. For each task, the output length is kept the same as in the original paper.
The outcomes are summarized in Table~\ref{tab:longbench}. We observe that slicing errors in RMSNorm and softmax accumulate with sequence length, leading to some degradation on LongBench. \textbf{TPLA (align)} follows the same alignment recipe as in the previous section, but its effectiveness is limited because the alignment corpus is formed by concatenating short texts. In contrast, \textbf{TPLA (PD-sep.)} adopts a prefill--decode separation: MLA is used unchanged in the prefill stage (no slicing of RMSNorm/softmax), and the resulting KV cache is partially reused by TPLA during decoding, which reduces first-token latency and accuracy loss. On DeepSeek-V2-Lite, the training-free \textbf{TPLA (PD-sep.)} surpasses the aligned variant, and on DeepSeek-V3 the model retains strong long-form reasoning with only a modest average drop of \(2.15\%\). These small losses, compared with training from scratch, are likely recoverable with a small amount of additional training.

\subsection{Ablation Study}
\label{sec:ablation_study}

\subsubsection{Part 1}

We highlight two structural differences. (i) \emph{Per-head latent capacity:} GLA gives each attention head only half of the latent dimension, whereas TPLA preserves the full latent dimension per head. (ii) \emph{Prefill–decode (PD) separation:} during the compute-intensive prefill stage we keep the reparameterized MLA form \emph{without} splitting RMSNorm or softmax; during decoding, TPLA uses PD separation while consuming the prefill KV cache.
We analyze the results in Table~\ref{tab:cs} to quantify these effects:

\textbf{1) MLA $\rightarrow$ GLA conversion.} Directly converting MLA to GLA forces each attention head to access only half of its original latent representation, causing substantial information loss and a marked accuracy drop across all benchmarks.

\textbf{2) Prefill--decode separation.} Avoiding RMSNorm/softmax partitioning in prefill reduces approximation error for the vast majority of tokens. Moreover, the MLA reparameterization enables the prefill KV cache to be used directly by TPLA at decode time, improving both quality and efficiency.

\begin{figure*}[ht]
	\centering
    \includegraphics[width=\textwidth]{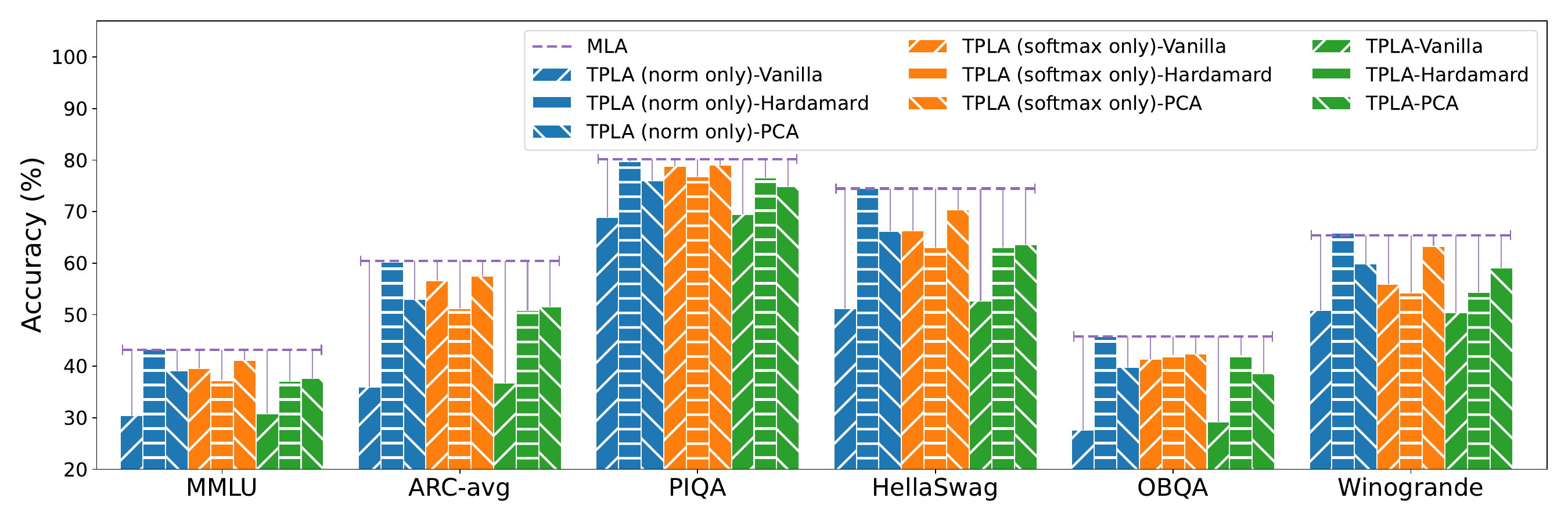}
    \caption{Accuracy across multiple benchmarks under different tensor-parallelism methods (indicated by colors) and reparameterization strategies (indicated by textures). 
    The purple horizontal line marks the original DeepSeek\mbox{-}V2\mbox{-}Lite accuracy, and the vertical bars show each method’s accuracy drop relative to this MLA baseline. \textbf{TPLA (norm only)} parallelizes RMSNorm across two devices, followed by an \texttt{allgather} before the softmax. 
    \textbf{TPLA (softmax only)} applies RMSNorm normally and parallelizes the softmax. 
    \textbf{TPLA} parallelizes both RMSNorm and softmax. 
    \textbf{Original} splits parameters evenly; 
    \textbf{Hadamard} balances parts prior to splitting; 
    \textbf{PCA} concentrates information into earlier dimensions before splitting. 
    To better visualize method-induced loss, TPLA results are reported without PD separation in this figure.}
    \label{fig:mla_gla_tpla}
\end{figure*}
\subsubsection{Part 2}

In Section \ref{sec:tpla}, we identified \textbf{RMSNorm} and \textbf{softmax} as the primary sources of error when converting MLA to TPLA. To mitigate this, we proposed two reparameterization strategies, \textbf{Hadamard-based} and \textbf{PCA-based}, to reduce the performance degradation introduced by parallelizing these components. This section presents ablation studies analyzing the impact of each reparameterization method on individual modules. Figure \ref{fig:mla_gla_tpla} presents ablation results analyzing the effectiveness of each method. The key findings are:

\textbf{1) Error ordering.} Empirically, slicing \textit{RMSNorm} incurs the least loss, slicing \textit{softmax} is worse, and slicing both is worst:

\textbf{2) TP on RMSNorm only}: The Hadamard-based method balances the norm computation across devices effectively, leading to performance comparable to the original MLA model on multiple tasks.

\textbf{3) TP on softmax only}: The PCA-based method concentrates information into the dimensions assigned to device 1, effectively preserving performance. In contrast, the Hadamard-based method fails to improve softmax accuracy. We hypothesize that the exponential nature of softmax makes it more sensitive to imbalance. Although Hadamard-based reparameterization achieves statistical balance across devices, small per-sample perturbations may result in significant asymmetries, adversely affecting final performance.

\textbf{4) TP over both RMSNorm and softmax.} When both components are parallelized, the PCA-based reparameterization consistently achieves the best performance. Consequently, we adopt this configuration for all experiments in the paper unless otherwise stated.

\subsection{Inference Speedup with TPLA}
\label{sec:speedup}

\subsubsection{Decoding Throughput}

LLM decoding is often \emph{memory-bandwidth bound}. TPLA splits each attention head’s input dimension across two devices, reducing per-device memory traffic and alleviating the bandwidth bottleneck. We evaluate the speedup of TPLA over MLA on two large models, \textbf{DeepSeek\mbox{-}V3\mbox{-}0324} (685B parameters) and \textbf{Kimi\mbox{-}K2\mbox{-}Base} (1T parameters). Because these models are extremely large and Mixture-of-Experts (MoE) routing can confound attention-speed effects, we remove MoE for timing. Both models are converted to \textbf{BF16}. All experiments use \textbf{FlashAttention-3} to ensure a fair comparison.

For \textbf{TPLA} with TP=2, the number of attention heads per device stays unchanged, while the latent dimension changes from $64{+}512$ to $(64{+}256)\times2$, so each device holds a 320-dimensional KV cache. For \textbf{MLA} with TP=2, the latent dimension is unchanged and heads are distributed across devices (e.g., DeepSeek-V3: $64$ heads $\times 2$; Kimi-K2: $32$ heads $\times 2$). For \textbf{MLA} with TP$>$2, we continue splitting along heads only. For \textbf{TPLA} with TP$>$2, we further split heads \emph{in addition to} halving the latent dimension; for example, with TP=4 on Kimi-K2-TPLA, we use $32$ heads $\times 2$ per device with a $320$-dimensional latent per head. In this setting, the per-device compute halves, while memory traffic matches TP=2; decoding remains memory-bound, so the speedup is similar to TP=2. Consequently, we report measurements on two GPUs.

Figure~\ref{fig:throughput} configures the maximum batch size at each context length. At a decoding length of $4096$, \textbf{TPLA} with $2d_h$ achieves up to \textbf{$\sim$2$\times$} the throughput of the single-head-latent \textbf{MLA} with $4d_h$, due to the smaller per-device KV cache. Our parallelization-friendly design raises peak throughput and is resilient under adverse serving loads. At a $32\mathrm{k}$ context length, \textbf{DeepSeek-TPLA} is \textbf{1.79$\times$} faster than MLA, and \textbf{Kimi-K2-TPLA} is \textbf{1.93$\times$} faster.

\begin{figure*}[ht]
	\centering
	\begin{subfigure}[b]{0.475\textwidth}
		\centering
		\includegraphics[width=0.95\textwidth]{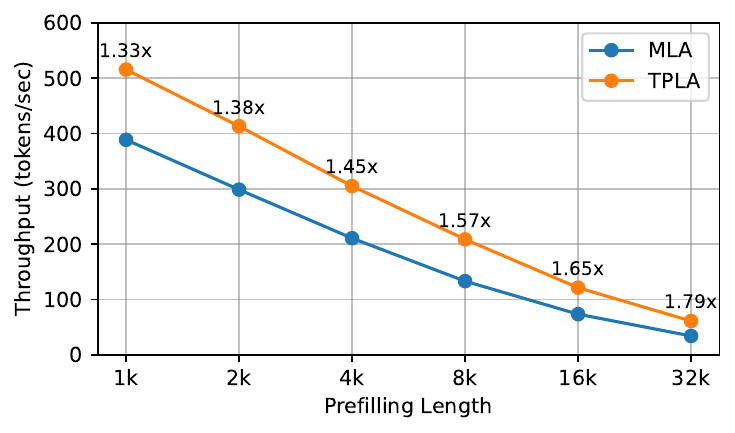}
		\caption{DeepSeek-V3-0324}
		\label{subfig:throughput}
	\end{subfigure}%
	\hfill
	\begin{subfigure}[b]{0.475\textwidth}
		\centering
		\includegraphics[width=0.95\textwidth]{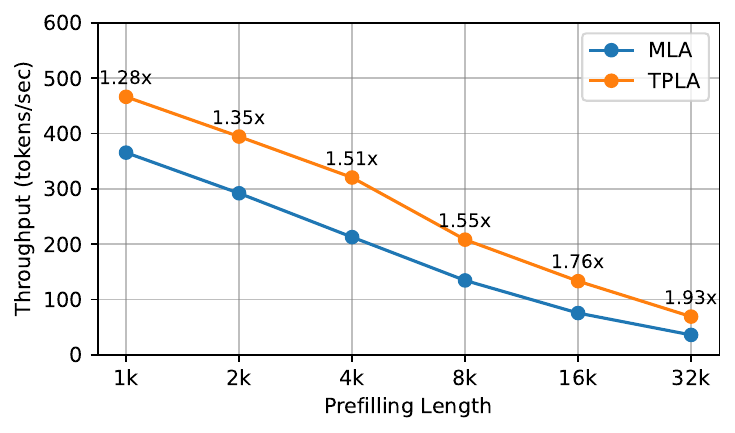}
		\caption{Kimi-K2-Base}
		\label{subfig:throughput_kimi}
	\end{subfigure}
	\caption{Throughout (Decoding) comparing MLA and TPLA.}
	\label{fig:throughput}
\end{figure*}

\begin{figure*}[ht]
	\centering
	\begin{subfigure}[b]{0.475\textwidth}
		\centering
		\includegraphics[width=0.95\textwidth]{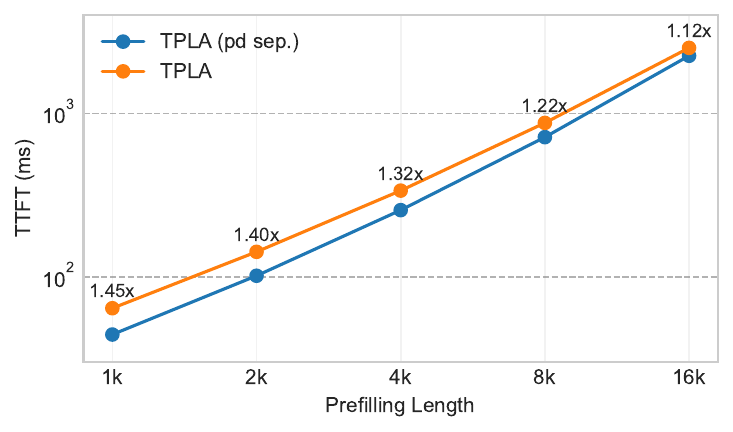}
		\caption{DeepSeek-V3-0324}
		\label{subfig:latency}
	\end{subfigure}%
	\hfill
	\begin{subfigure}[b]{0.475\textwidth}
		\centering
		\includegraphics[width=0.95\textwidth]{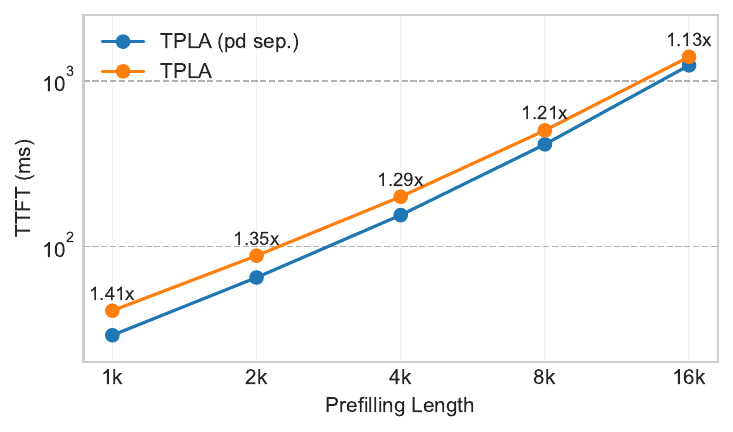}
		\caption{Kimi-K2-Base}
		\label{subfig:latency_kimi}
	\end{subfigure}
	\caption{Latency (TTFT) comparing TPLA and TPLA (pd sep.).}
	\label{fig:prefill_latency}
\end{figure*}

\subsubsection{Prefilling Latency}

The \emph{prefilling} stage of LLM inference is \emph{compute-bound}. Under TPLA’s TP separation, each device retains the original number of heads, whereas MLA can reduce heads per device by splitting them across devices. As a result, the original TPLA is not ideal for the compute-bound prefill stage. To address this, we introduce \textbf{TPLA (sep.)}: it applies the same reparameterization to MLA but \emph{does not} slice RMSNorm or softmax, thereby introducing no approximation error. During prefill, the structure matches MLA: under TP we do not change the latent dimension but partition heads across devices. This significantly reduces per-device compute and alleviates the compute bottleneck.

Figure~\ref{fig:prefill_latency} reports \textbf{TTFT} (Time to First Token) on two GPUs for MoE-removed DeepSeek\mbox{-}V3\mbox{-}0324 and Kimi\mbox{-}K2\mbox{-}Base. At a 1K prompt length, \textbf{TPLA (sep.)} is \textbf{1.4$\times$} faster than TPLA for both models. Given its accuracy-friendly design, this \emph{1.4$\times$} gain is essentially a “free lunch.”
\section{Conclusion, Limitation and Future Work}
\label{sec:conclusion}
We introduce TPLA, which combines the KV cache compression efficiency of MLA with strong compatibility for Tensor Parallelism. It can directly inherit checkpoints from MLA-pretrained models. With two proposed reparameterization techniques, it substantially reduces the loss incurred by converting the attention formulation; combined with PD separation, the training-free conversion error can be driven to a very small level. We evaluate TPLA on commonsense reasoning tasks and the more challenging LongBench benchmark, finding that it preserves the original model's performance well. Extensive ablations confirm the effectiveness of our TP slicing and reparameterization designs. TPLA achieves up to 2$\times$ improvement in throughput, and PD separation delivers up to 29\% latency reduction. Overall, TPLA shows strong potential as a powerful and efficient replacement for MLA.

 \paragraph{Limitation and Future Work.} 
 Although PCA demonstrates better performance over Hadamard transform, it has inherent limitations. Specifically, PCA concentrates most of the data’s informative content in the first few dimensions, which provide a representative summary of the global structure. In contrast, the later dimensions primarily capture negligible noise and minor variations that contribute minimally to the overall representation. Consequently, TPLA with group-partitions $g = 2$ can achieve good performance, but when $g > 2$, it probably fails to maintain effectiveness. By contrast, numerical-value balancing via orthogonal transforms, particularly the Hadamard transform, tends to be more effective when partitioning into multiple groups. Empirically, inserting a Hadamard transform into the RMSNorm slicing part yields almost no performance degradation. In future work, we will design and evaluate optimized Hadamard-like orthogonal matrices to balance softmax slicing, thereby improving both robustness and scalability. One advantage of TPLA is that it can directly inherit MLA checkpoints, but this also introduces some conversion errors. Our experiments fully validate TPLA’s expressive capacity and speed advantages. In future work, we will post-pretrain DeepSeek-V3, or train a TPLA-based model from scratch, to further demonstrate TPLA’s excellent expressiveness.

\bibliographystyle{unsrt}  
\bibliography{references}

\end{document}